\documentclass{article}
  \newdimen\paravsp  \paravsp=1.3ex 

\usepackage{latexsym,amsmath,amssymb}
\usepackage{hyperref} 
\usepackage{ntheorem} 
\usepackage[capitalize]{cleveref} 
\usepackage[usenames]{color} 
\usepackage{authblk}

\DeclareMathOperator{\argmax}{argmax}
\DeclareMathOperator{\dual}{dual}
\DeclareMathOperator{\env}{env}

\def\,{\mskip 3mu} \def\>{\mskip 4mu plus 2mu minus 4mu} \def\;{\mskip 5mu plus 5mu} \def\!{\mskip-3mu}

\def\textmuskip{\thinmuskip= 0mu                    \medmuskip=  1mu plus 1mu minus 1mu \thickmuskip=2mu plus 3mu minus 1mu}
\textmuskip

\newtheorem{theorem}{Theorem}

\newtheorem{conjecture}[theorem]{Conjecture}

\newtheorem{definition}[theorem]{Definition}
\newtheorem{example}[theorem]{Example}

\newenvironment{keywords}{\centerline{\bf\small
Keywords}\begin{quote}\small}{\par\end{quote}\vskip 1ex}
\newenvironment{proof}{{\noindent\bf Proof.}}{\qed\vskip 1ex}

\newtheorem{myexample}[theorem]{Example} 

\def\paradot#1{\vspace{\paravsp plus 0.5\paravsp minus 0.5\paravsp}\noindent{\bf\boldmath{#1.}}} 
\def\hrefurl#1{\href{#1}{\rule{0ex}{1.7ex}\color{blue}\underline{\smash{#1}}}} 

\def\eps{\varepsilon}           
\def\epstr{\epsilon}            
\def\qed{\hspace*{\fill}\rule{1.4ex}{1.4ex}$\quad$\\} 
\def\eoe{\hspace*{\fill} $\blacklozenge\quad$} 
\def\eor{\hspace*{\fill} {\LARGE\textbullet}$\quad$} 
\AtEndEnvironment{example}{\eoe\endexample}
\AtEndEnvironment{remark}{\eor\endremark}
\def\SetR{\mathbb{R}}           
\def\SetN{\mathbb{N}}           
\def\SetB{\mathbb{B}}           

\def\cA{{\cal A}}                
\def\cE{\cal{E}}                 
\def\cM{\cal{M}}                 
\def\cX{{\cal X}}               

\def\geqx{\;\stackrel{\times}{\geq}\;}

\def\eqx{\;\stackrel{\times}{=}\;}
\def\ngeqx{\;\stackrel{\times}{\ngeq}\;}

\def\h{\text{\ae}} 
\def\lscsemi{\cM^\text{semi}_\text{lsc}}
\def\lscccs{\cM^\text{ccs}_\text{lsc}}

\def\jointxi{\xi^U} 

\begin{document}


\title{\bf\Large\hrule height5pt \vskip 4mm 
Formalizing Embeddedness Failures \\ in Universal Artificial Intelligence
\vskip 4mm \hrule height2pt}
\author[1]{Cole Wyeth}
\author[2]{Marcus Hutter}
\affil[1]{\normalsize Cheriton School of Computer Science \\
\normalsize University of Waterloo\\
\normalsize 200 University Ave W, Waterloo, ON N2L 3G1, Canada \\
\normalsize \hrefurl{https://colewyeth.com/}
}
\affil[2]{\normalsize 
Google DeepMind \& Australian National University\\
\normalsize \hrefurl{http://www.hutter1.net/}
}
\date{April 2025}
\maketitle

\begin{abstract}
We rigorously discuss the commonly asserted failures of the AIXI reinforcement learning agent as a model of embedded agency. We attempt to formalize these failure modes and prove that they occur within the framework of universal artificial intelligence, focusing on a variant of AIXI that models the joint action/percept history as drawn from the universal distribution. We also evaluate the progress that has been made towards a successful theory of embedded agency based on variants of the AIXI agent.   
\vspace{5ex}\def\contentsname{\centering\normalsize Contents}\setcounter{tocdepth}{1}
{\parskip=-2.7ex\tableofcontents}
\end{abstract}


\newpage
\section{Introduction}\label{sec:Intro}

The original AIXI reinforcement learning agent, intended as a nearly parameter-free formal gold standard for artificial general intelligence (AGI), is a \emph{Cartesian dualist} that believes it is interacting with an environment from the outside, in the sense that its \emph{policy} is fixed and not overwritten by anything that happens in the environment, though its \emph{actions} can certainly adapt based on the percepts it receives. This is frequently compared to a person playing a video game, who certainly does not believe he is being simulated by the game but rather interacts with it only by observing the screen and pressing buttons. In contrast, it would presumably be important for an AGI to be aware that it exists within its environment (the universe) and its computations are therefore subject to the laws of physics. With this in mind, we investigate versions of the AIXI agent \cite{hutter_theory_2000} that treat the action sequence $a$ on a similar footing to the percept sequence $e$, meaning that the actions are considered as explainable by the same rules generating the percepts. The most obvious idea is to use the universal distribution to model the joint (action/percept) distribution (even though actions are selected by the agent). Although this is the most direct way to transform AIXI into an embedded agent, it does not appear to have been analyzed in detail; in particular, it is usually assumed (but not proven) to fail (often implicitly, without distinguishing the universal sequence and environment distributions, e.g. \cite{fallenstein_reflective_aixi_2015}).

\paradot{Outline}
First, we discuss the more sophisticated approaches to embedded AIXI-like agents as points of reference. Then we give a highly compressed introduction to the notions of algorithmic information theory (AIT) and particularly universal artificial intelligence (UAI) needed to model universal mixtures of sequence distributions and environments, followed by some mappings that relate sequence and environment distributions. Finally, we introduce some positive and some negative results for our embedded AIXI variant (called joint AIXI), each of which is perhaps a little surprising, but follows easily from recent progress in AIT.   

\paradot{Terminology for the problem setting}
AIXI has a harder incomputability level than the environments in its hypothesis class, since its actions are not sampled from its belief distribution. This means that we are analyzing an \emph{unrealizable} situation, where the interaction history is generated by a process outside the hypothesis class. For the purposes of learning the joint distribution, it is hard to find any useful guarantees on the action sequence, so we will treat it as \emph{adversarially} chosen and consider the worst case.  

\section{Related Work}\label{sec:related_work}

Several sophisticated embedded versions of AIXI have been proposed. 

\paradot{Reflective oracles}
By expanding AIXI's hypothesis class to include machines with access to a special type of ``reflective'' oracle \cite{fallenstein_reflective_oracles_2015}, researchers at the Machine Intelligence Research Institute (MIRI)  were able to construct a version of AIXI that is at the same incomputability level as the environments in its hypothesis class \cite{fallenstein_reflective_aixi_2015}. This reflective AIXI faces a \emph{realizable} learning problem, so there is no reason to view its own actions as adversarially chosen. In fact, reflective oracle machines can directly compute the conditional probabilities of actions and percepts (not just their joint distributions) so that the perspective change between sequence and environment distribution becomes trivial. We believe that reflective AIXI is an excellent (and perhaps underappreciated) approach to embeddedness, primarily because there is a limit-computable reflective oracle \cite{leike_formal_2016}, which means that reflective AIXI has a stochastic anytime algorithm. This is a better computability result than has been demonstrated for AIXI, which may be as hard as $\Delta^0_3$ without $\eps$-approximation \cite{leike_computability_2018}, meaning that reflective AIXI not only addresses embeddedness concerns but actually suggests an effective AIXI approximation. The main (serious) limitation of reflective AIXI is that it is still a computationally unbounded model of intelligence, so it is not clear how it treats recursive self-improvement (this is an interesting research question). 

\paradot{Self-AIXI}
The Self-AIXI agent \cite{catt_self-predictive_2023} functions in a similar way to reflective AIXI, but it is uncertain of its own policy as well as its environment. The framework in the paper does not specify particular belief distributions for either of these, but asks for choices that make an optimal policy lie in the policy class, a realizability assumption. It is easy to find such a choice by taking advantage of reflective oracles (this is proven in a forthcoming paper of ours, \cite{wyeth_hutter_reflective_2025}). Self-AIXI plans only one step ahead to locally maximize its action-value function; the variation of our joint AIXI based on this strategy can easily be analyzed by the same methods introduced in this paper, though our focus on percept prediction becomes less justifiable. 

\paradot{Space-time embedded intelligence}
Orseau and Ring depart more drastically from the original AIXI model but continue to take advantage of the tools of UAI by proposing various degrees of embeddedness of an agent inside the universe's computation \cite{orseau_space-time_2012}. In their most extreme model, the agent is only a collection of bits on the environment machine's tape. Although this seems to capture all problems of embeddedness, it also makes so few assumptions that we are not convinced that it is a useful theory of intelligence.

\paradot{Embedded decision theories}
More extreme departures from the UAI framework such as Infra-Bayesian Physicalism \cite{kosoy_infra-bayesian_2021} and updateless decision theory \cite{dai_towards_2009} are beyond the scope of this paper. Our focus is on the most direct modification of AIXI for embedded agency, which seems seriously neglected in the literature. 

\section{Mathematical Preliminaries}\label{sec:math_prelim}

\paradot{Notation}
For any finite alphabet $\Sigma$, we denote the set of finite strings over $\Sigma$ as $\Sigma^*$ and the set of infinite sequences over $\Sigma$ as $\Sigma^\infty$. If $s \in \Sigma^*$ or $s\in\Sigma^\infty$, then $s_i \in \Sigma$ is the $i^\text{th}$ symbol of $s$, indexing from 1. Similarly $s_{i:j}$ for $i \leq j$ is the substring from indices $i$ to $j$. Particular alphabets we will discuss are the set of actions $\cA$ available to an agent and the set of percepts $\cE$ that the environment $\nu$ might produce and send to the agent. A percept consists of an observation in $\mathcal{O}$ and a reward in $\mathcal{R} \subset \mathbb{R}$. The action/percept at time $t$ will be denoted $a_t/e_t$, and the history of actions and observations before time $t$ will be written $\h_{<t} = a_1e_1...a_{t-1}e_{t-1}$. 

\begin{definition}[lower semicomputable]
A function $f$ is \emph{lower semicomputable} (l.s.c.) if there is a computable function $\phi(x,k)$ monotonically increasing in its second argument with $\lim_{k\to\infty} \phi(x,k) = f(x)$. That is, $f$ can be approximated from below.  
\end{definition}
In AIT the history distribution often has a probability gap because interaction may terminate at a finite time. For this reason, it is modeled by a ``semimeasure,'' not a proper probability measure. 

\begin{definition}[semimeasure]
A semimeasure\footnote{Technically, this only defines a pre-semimeasure, but a unique \emph{sensible} extension to the generated $\sigma$-algebra exists \cite{wyeth_hutter_semimeasures_2025}.} $\nu$ is a function $\Sigma^*\to\mathbb{R}^+$ satisfying $\nu(x) \geq \sum_{a \in \Sigma} \nu(xa)$. 
\end{definition}
For our purposes, semimeasures always assign probability $\leq 1$ to the empty string $\epstr$. The set of such l.s.c.\ semimeasures is denoted $\lscsemi$.

\begin{definition}[$\xi_U$]\label{def:universal_mixture}
    The universal distribution $\xi_U$ is defined as 
    \begin{equation}\label{eq:xi_U}
        \xi_U(\h_{1:t}) :=  \sum_{\nu_i\in\lscsemi} w_i \nu_i(\h_{1:t})
    \end{equation}
    for $\h\in(\cA\times\cE)^\infty$, where $i\mapsto w_i>0$ is a l.s.c.\ function with $\sum_i w_i\leq 1$, e.g.\ $w_i:=[i(i+1)]^{-1}$.
    An alternative construction is
    \begin{equation}
        \xi_U(x) \stackrel{\times}{=}\; \sum_{p: U(p) = x*} 2^{-l(p)}
    \end{equation}
    with monotone UTM $U$ (a ``joint'' distribution producing sequences which are NOT action contextual).
\end{definition}
For simplicity of exposition we assume $\cA = \cE$ by expanding the smaller alphabet. 

\paradot{Chronological semimeasures}
We write $\nu^\cdot(e_{1:t} || a_{1:t})$ to denote the probability that the environment $\nu^\cdot$ produces percepts $e_{1:t}$ when the agent takes actions $a_{1:t}$. Formally, this notation is used whenever $\nu^\cdot(\cdot||\cdot)$ is a two-argument function satisfying the chronological semimeasure condition $\nu^\cdot(e_{1:t} || a_{1:t}) \geq \sum_{e_{t+1}} \nu^\cdot(e_{1:t+1}||a_{1:t+1})$. Despite the name, when the two arguments are treated as forming one interaction history $\h_{1:t}$, this is actually a weaker requirement than ordinary semimeasures must satisfy. Usually, the superscript $\cdot$ is replaced by some index or identifying symbol, distinguishing environments from sequence distributions, which use subscripts when any adornment is required.  

\paradot{Semimeasure representation}
It is clear that for any semimeasure $\nu_\cdot$, the map $e_{1:t},a_{1:t} \rightarrow  \prod_{i=1}^t \nu_\cdot(e_i|\h_{<i}a_i)$ defines a chronological semimeasure that we will write as $\nu^\cdot(e_{1:t}|| a_{1:t})$. Conversely, if $\nu^\cdot$ is a chronological semimeasure, we can choose $\nu_\cdot$ so that $\nu_\cdot(a_i|\h_{<i}) = 1/|\cA|$ (or any arbitrary element of $\Delta \cA$) and $\nu_\cdot(e_i|\h_{<i}a_i) = \nu^\cdot(e_{1:i}||a_{1:i})/\nu^\cdot(e_{<i}||a_{<i})$ (and $\nu_\cdot$ has 0 probability of producing a percept in an action position or vice versa). Then $\nu_\cdot$ obviously satisfies the semimeasure property at action positions and satisfies the semimeasure property at percept positions by the chronological semimeasure property of $\nu^\cdot$, and it is easy to see that $\nu^\cdot(e_{1:t}||a_{1:t}) = \prod_{i=1}^t \nu_\cdot(e_i|\h_{<i}a_i)$. Satisfying the chronological semimeasure property is equivalent to having such a semimeasure representation; unfortunately, this representation result does not seem to hold when we restrict to \emph{l.s.c.} (chronological) semimeasures\footnote{The related fact that the conditionals of an l.s.c. semimeasure may not be l.s.c. follows easily from a diagonalization argument for $\xi_U$. It seems harder to find a clean example that \emph{clearly} shows the chronological semimeasure on half of the bit positions, induced by $\env$ of an l.s.c. semimeasure, is not l.s.c.}. 

\begin{definition}[$\xi^\text{AI}$]\label{def:universal_environment}
    The universal chronological semimeasure $\xi^\text{AI}$ is defined by 
    \begin{equation}\label{eq:xiai}
        \xi^\text{AI}(e_{1:t}||a_{1:t}) := \sum_{\nu^i \in \lscccs} w_i \nu^i(e_{1:t}||a_{1:t})
    \end{equation}
    with $w_i$ as in \cref{def:universal_mixture}.
    Alternatively, we can obtain $\xi^\text{AI}$ by 
    \begin{equation}
        \xi^\text{AI}(e_{1:t}||a_{1:t}) \stackrel{\times}{=}\; \sum_{p: U^C(p,a_{1:t}) = e_{1:t}} 2^{-l(p)} 
    \end{equation}
    where  ``chronological'' UTM $U^C$ only reads actions up to time $t$ before producing $e_t$.   
\end{definition}
Note that when $\xi^\text{AI}$ is viewed as a function of $\h_{1:t}$, it is not a semimeasure because it is not subadditive at action indices. 

\begin{definition}[domination]
A semimeasure $\nu$ (multiplicatively) dominates a semimeasure $\mu$, written $\nu \;\stackrel{\times}{\geq}\; \mu$, if $\exists c \in \mathbb{R}^+$ such that $\forall x \in \Sigma^*$, $\nu(x) \geq c\mu(x)$.  
\end{definition}
For chronological semimeasures dominance requires the above to hold for each action sequence. Multiplicative dominance is stronger than absolute continuity, which is the usual criterion for merging-of-opinions style results \cite{blackwell_merging_1962}. 

For our purposes, chronological semimeasures will represent either an environment or an agent's policy. An agent's goal is to maximize its cumulative (often discounted) reward by choosing an optimal sequence of actions. The optimal policy is defined as
\begin{equation}
    \begin{split}
        \pi^*_\nu &= \argmax_\pi V^\pi_\nu \\ 
        &:= \argmax_\pi \sum_{t=1}^m \sum_{\h_{1:t}} r_t\pi(a_{1:t}||e_{1:t-1})\nu(e_{1:t}||a_{1:t}) \\
    \end{split}
\end{equation}
When $\pi$ is deterministic, we can abuse notation by treating it as a function from histories to actions. Then, in the case that $\mu$ is a measure, we can write the constraint\footnote{Technically this constraint may be ignored off-policy.} on the optimal policy's action choice explicitly as
\begin{equation}
    \pi^*_\nu(\h_{1:t}) \in \argmax_{a_{t+1}} \sum_{e_{t+1}} ... \max_{a_m} \sum_{e_m} \nu(e_{1:m}||a_{1:m}) \sum_{i=t+1}^m r_t
\end{equation}
With discounting, this can also be extended to infinite horizons, but the distinction will not be important for us.

\section{The Dualistic Mixture}

We introduce maps $\env$ and $\dual$ that (respectively) convert a (semimeasure) history distribution to an environment and combine a policy and environment to get a history distribution, then apply these ideas to introduce joint AIXI in the next section.   

An observation we will find useful is that given any pair of l.s.c.\ chronological semimeasures $\nu$ generating percepts and $\pi$ generating actions, the history distribution $\nu^\pi$ that they induce is an l.s.c.\ semimeasure. To avoid interfering with (super/sub)scripts, we usually write $\dual(\nu,\pi) := \nu^\pi$. There is a semimeasure that encodes the assumption that actions are generated by a l.s.c.\ agent and percepts by an l.s.c.\ environment:
\[
\xi_{\dual} := \sum_{\nu, \pi \in \cM^\text{semi}_\text{lsc}} w^\pi_\nu \dual(\nu, \pi)
\]
We will assume that $w^\pi_\nu = \omega_\pi w_\nu=2^{-K(\pi)}2^{-K(\nu)}$ (agent and environment are independent). This is the assumption made by Self-AIXI which makes it less general than our joint AIXI. It is immediate that $\xi_U \smash{\geqx} \xi_{\dual}$.
Let 
\[
\env(\nu)(e_{1:t}||a_{1:t}) := \prod_{i=1}^t \nu(e_i|\h_{<i}a_i)
\]
To be well-defined, this requires $\nu > 0$, which holds for $\xi_U$. 
When $w^\pi_\nu = \omega_\pi w_\nu$, factoring yields $\xi_{\dual} = \dual(\sum_\pi  \omega_\pi \pi^\cdot, \sum_\nu w_\nu \nu^\cdot)$ so $\env(\xi_{\dual}) = \sum_\nu w_\nu \nu^\cdot \eqx \xi^\text{AI}$. 

\section{Joint AIXI}
Now we introduce our joint AIXI model which uses the environment version of $\xi_U$ to model the joint distribution. In particular we investigate its relationship to the conventional $\xi^\text{AI}$. Let $\jointxi := \env(\xi_U)$. Then
\begin{equation}
    \begin{split}
        \jointxi(e_t|\h_{<t}a_t) &= \xi_U(e_t|\h_{<t}a_t) \\
        &= \frac{\sum_i w_i \nu_i(\h_{<t}a_te_t)}{\xi_U(\h_{<t}a_t)} \\
        &= \sum_{\nu} w_i(\h_{<t}a_t) \nu_i(e_t|h_{<t}a_t) \\
        &= \sum_{\nu} w_i(\h_{<t}a_t) \nu^i(e_t|h_{<t}a_t) \\
    \end{split}
\end{equation}
where
\begin{equation}\label{eq:posteriors}
    w_i(\h_{<t}a_t) := \frac{\nu_i(\h_{<t}a_t)}{\xi_U(\h_{<t}a_t)}
\end{equation}

so $\jointxi$ is not a linear combination of $\nu^i$ because of the way that the actions control the weights; $\jointxi$ encodes a kind of sequential action evidential decision theory \cite{everitt_sequential_2015}, because semimeasures that prefer the action sequence $a$ are weighted more heavily as environments in $\jointxi$.
Joint AIXI is defined as a Bayes-optimal policy for this belief distribution\footnote{Arguably, this decision rule is somewhat short-sighted: despite planning ahead, it does not condition on its intended future decisions - that is, $\xi^U(e_t|\h_{<t}a_t)$ does not depend on $a_{t+1:\infty}$. In contrast $\xi^\text{AI}_\text{alt}$ of \cite{hutter_universal_2005}, when paired with the iterative value function, can be viewed as attempting to update on intended actions in advance, but actually fails to meet the conditions of a chronological semimeasure (our unpublished result).}:
\begin{equation}
    \pi^\text{JAIXI} := \pi^*_{\jointxi}
\end{equation}

\paradot{An aside}
An arguably more natural definition is
\[
\xi^{\env(U)}(e_{1:t}||a_{1:t}) := \sum_\nu w_\nu \env(\nu)(e_{1:t}||a_{1:t}) \geqx \env(\xi_U)(e_{1:t}||a_{1:t}) 
\]
Another way of writing this is
\[
\xi^{\env(U)}(e_t|\h_{<t}a_t) = \sum_\nu w_\nu(e_{<t}||a_{<t}) \nu(e_t|h_{<t}a_t) 
\]
where $w_\nu(e_{<t}||a_{<t}) = \env(\nu)(e_{<t}||a_{<t})/\xi^{\env(U)}(e_{<t}||a_{<t})$. Note that $w( \cdot || \cdot )$ is not a strictly correct use of the $||$ notation, but only indicates a ratio of chronological semimeasures. $\xi^{\env(U)}$ is not the same as $\env(\xi_U)$, but in fact dominates $\env(\xi_U)$. Also, for any l.s.c.\ policy $\pi$, $\dual(\xi^\text{AI}, \pi)$ is an l.s.c.\ semimeasure, so
\[
\xi^{\env(U)}(e_{1:t}||a_{1:t}) \geqx \env(\dual(\xi^\text{AI}, \pi))(e_{1:t}||a_{1:t}) = \xi^\text{AI}(e_{1:t}||a_{1:t}) 
\]
so $\xi^{\env(U)}$ also dominates the explicitly causal $\xi^\text{AI}$! This makes it a rather fascinating mixture, which seems to give some weight to both CDT and EDT, but to judge them by the standards of CDT when updating those weights. 

\paradot{Adversarial learning}
We can adapt results on the prediction of selected bits directly from \cite{lattimore_universal_selected_bits_2011} to understand the relationship between $\jointxi$ and $\xi^\text{AI}$. The authors investigated whether computable structure at certain computable indices of a sequence can be learned by the universal distribution even when the rest of the sequence is adversarially chosen. They were motivated by supervised learning, where the distribution of examples may be much more complicated/unpredictable than the distribution of labels. Our motivation is different: we expect the even = percept indices to come from a l.s.c.\ environment distribution (given the odd = action indices as context), but this simple partition of the indices matches the example case in \cite[Problem 2.9]{hutter_universal_2005}.  

\paradot{Failure to learn a simple environment}
According to \cite[Theorem 12]{lattimore_universal_selected_bits_2011}, it is possible for $\xi_U$ to fail to predict a binary sequence at even indices despite even bits exactly matching the preceding bits at odd indices. Formally (translated to our notation),

\begin{theorem}[Adversarial non-convergence of $\xi_U$]\label{thm:adv_non_conv}
    There exists $\omega\in\SetB^\infty$ with $\omega_{2n} = \omega_{2n-1}$ but $\liminf_n \xi_U(\omega_{2n}|\omega_{1:2n-1}) < 1$. 
\end{theorem}
Such $\omega$ must not be computable, but the theorem is still surprising since the even bits are very easy to predict for a human. Now consider the simplest possible environment $\mu^\text{id}$ with binary action space, empty observation space, and binary reward space, defined by
\begin{equation}
    \mu^\text{id}(e_t|\h_{<t}a_t) = [\![e_t = a_t]\!]
\end{equation}
Clearly, this is a l.s.c.\ chronological semimeasure. 

\begin{theorem}[Adversarial non-convergence of $\jointxi$]\label{thm:xi_U_non_conv}
    There exists $a \in \SetB^\infty$ such that $\jointxi(a_{1:t}||a_{1:t}) \rightarrow 0$ as $t \rightarrow \infty$.
\end{theorem}
\begin{proof}
    This is a direct result of \cref{thm:adv_non_conv}.
\end{proof}

\begin{theorem}
    $\jointxi \smash{\ngeqx} \xi^\text{AI}$
\end{theorem}

\begin{proof}
    $\xi^\text{AI}(a_{1:t}||a_{1:t}) \smash{\geqx} \mu_\text{id}(a_{1:t}||a_{1:t}) = 1$,
    while $\jointxi(a_{1:t}||a_{1:t})\to 0$.
\end{proof}
We expect that domination fails in the other direction as well because $\xi^U$'s posteriors (as in \cref{eq:posteriors}) treat $a$ much differently than $\xi^\text{AI}$'s posteriors, which should sometimes be advantageous for prediction. 

\begin{conjecture}
    $\xi^\text{AI} \smash{\ngeqx} \jointxi$
\end{conjecture}

\paradot{Normalization allows learning computable environments}
Recall that $\xi_U$ is only a semimeasure and not a proper probability measure. The most common ``normalization'' or completion of $\xi_U$ to a measure is called Solomonoff normalization:
\begin{equation}
    \hat{\xi}_U(\omega_t|\omega_{<t}) = \frac{\xi_U(\omega_t|\omega_{<t})}{\sum_{\omega_t'\in\Sigma} \xi_U(\omega_t'|\omega_{<t})}
\end{equation}
Surprisingly, applying this normalization allows a type of learning in the adversarial case. We translate \cite[Theorem 10]{lattimore_universal_selected_bits_2011} as follows:

\begin{theorem}[Adversarial learning for $\hat{\xi}_U$]\label{thm:adv_learning}
    Let $f : \SetB^* \rightarrow \SetB \cup \{\epstr\}$ be a total recursive function and consider $\omega\in\SetB^\infty$ satisfying $\omega_n = f(\omega_{<n})$ when $f(\omega_{<n}) \neq \epstr$. Then for any infinite sequence $n_1,n_2,...$ with $f(\omega_{n_i}) \neq \epstr$, $\lim_{i\rightarrow\infty} \hat{\xi}_U(\omega_{n_i}|\omega_{<n_i}) = 1$.
\end{theorem}

In other words, if the bits at certain recursively checkable indices are a recursive function of the preceding sequence, $\hat{\xi}_U$ will eventually learn to make good predictions at those indices. Unfortunately, this positive result seems unlikely to generalize to the stochastic case (though this is an open problem).

\begin{theorem}[Adversarially learning environments]\label{thm:learning_det_env}
    If $\mu^\cdot \in \lscccs$ is deterministic and $e \sim \mu^\cdot(\cdot|a)$ with $l(e) = \infty$, then $\lim_{t\rightarrow\infty} \env(\hat{\xi}_U)(e_t|\h_{<t}a_t) = 1$. 
\end{theorem}

\begin{proof}
    A deterministic environment that is l.s.c.\ must also be recursive in the sense that the next percept is finitely computable from the history (it must also be a measure given this action sequence for the history to be infinite). If the action or percept spaces are not binary, choose a fixed binary encoding, and these statements also apply at the bit level. Observe that $\env(\hat{\xi}_U)(e_t|\h_{<t}a_t) = \hat{\xi}_U(e_t|\h_{<t}a_t)$ by definition and apply \cref{thm:adv_learning}.   
\end{proof}

\section{Discussion}\label{sec:disc}

The presumed failure of $\pi^\text{JAIXI}$ was the motivation for studying reflective AIXI, which neatly resolves the problems in this paper. We might expect joint AIXI's learning to sometimes fail because the history is not sampled from its belief distribution. Though we prove in this paper that it fails to converge to the correct answer in a simple case, we assume adversarially selected action bits. For a deployed agent, the action bits would be selected by the policy $\pi^\text{JAIXI}$ which may never produce these adversarial action sequences, so we do not know whether $\pi^\text{JAIXI}$ learns to behave well in reasonable environments. On the other hand, we prove that normalizing the joint distribution allows learning deterministic environments. To match the results for AIXI, we would actually like (on-policy) \emph{fast} convergence results against all l.s.c.\ chronological semimeasures. Both our positive and our negative results leave interesting open problems. The technical difficulties involved in establishing anything about joint AIXI justify Hutter's nontrivial choice to invent a universal mixture for l.s.c.\ \emph{chronological} semimeasures as a basis for studying Cartesian agents. 

\paradot{Acknowledgements}
This work was supported in part by a grant from the Long-Term Future Fund (EA Funds - Cole Wyeth - 9/26/2023).


\bibliographystyle{alpha}

\begin{thebibliography}{1}
\bibitem[BD62]{blackwell_merging_1962}
D.~Blackwell and L.~Dubins.
\newblock Merging of opinions with increasing information.
\newblock {\em Annals of Mathematical Statistics}, 33:882--887, 1962.

\bibitem[CGMH23]{catt_self-predictive_2023}
Elliot Catt, Jordi Grau-Moya, Marcus Hutter, Matthew Aitchison, Tim Genewein,
  Grégoire Delétang, Kevin Li, and Joel Veness.
\newblock Self-{Predictive} {Universal} {AI}.
\newblock {\em Advances in Neural Information Processing Systems},
  36:27181--27198, December 2023.

\bibitem[Dai09]{dai_towards_2009}
Wei Dai.
\newblock Towards a {New} {Decision} {Theory}.
\newblock August 2009.

\bibitem[ELH15]{everitt_sequential_2015}
Tom Everitt, Jan Leike, and Marcus Hutter.
\newblock Sequential extensions of causal and evidential decision theory, 2015.

\bibitem[FST15]{fallenstein_reflective_aixi_2015}
Benja Fallenstein, Nate Soares, and Jessica Taylor.
\newblock Reflective {Variants} of {Solomonoff} {Induction} and {AIXI}.
\newblock In Jordi Bieger, Ben Goertzel, and Alexey Potapov, editors, {\em
  Artificial {General} {Intelligence}}, pages 60--69, Cham, 2015. Springer
  International Publishing.

\bibitem[FTC15]{fallenstein_reflective_oracles_2015}
Benja Fallenstein, Jessica Taylor, and Paul~F. Christiano.
\newblock Reflective {Oracles}: {A} {Foundation} for {Classical} {Game}
  {Theory}, August 2015.
\newblock arXiv:1508.04145 [cs].

\bibitem[HM07]{hutter_semimeasures_2007}
Marcus Hutter and Andrej Muchnik.
\newblock On semimeasures predicting {Martin}-{Löf} random sequences.
\newblock {\em Theoretical Computer Science}, 382(3):247--261, September 2007.

\bibitem[Hut00]{hutter_theory_2000}
Marcus Hutter.
\newblock A {Theory} of {Universal} {Artificial} {Intelligence} based on
  {Algorithmic} {Complexity}, April 2000.
\newblock arXiv:cs/0004001.

\bibitem[Hut05]{hutter_universal_2005}
Marcus Hutter.
\newblock {\em Universal {Artificial} {Intellegence}}.
\newblock Texts in {Theoretical} {Computer} {Science} {An} {EATCS} {Series}.
  Springer, Berlin, Heidelberg, 2005.

\bibitem[Kos21]{kosoy_infra-bayesian_2021}
Vanessa Kosoy.
\newblock Infra-{Bayesian} physicalism: a formal theory of naturalized
  induction, November 2021.

\bibitem[LH18]{leike_computability_2018}
Jan Leike and Marcus Hutter.
\newblock On the computability of {Solomonoff} induction and {AIXI}.
\newblock {\em Theoretical Computer Science}, 716:28--49, March 2018.

\bibitem[LHG11]{lattimore_universal_selected_bits_2011}
Tor Lattimore, Marcus Hutter, and Vaibhav Gavane.
\newblock Universal {Prediction} of {Selected} {Bits}.
\newblock In Jyrki Kivinen, Csaba Szepesvári, Esko Ukkonen, and Thomas
  Zeugmann, editors, {\em Algorithmic {Learning} {Theory}}, pages 262--276,
  Berlin, Heidelberg, 2011. Springer.

\bibitem[LTF16]{leike_formal_2016}
Jan Leike, Jessica Taylor, and Benya Fallenstein.
\newblock A formal solution to the grain of truth problem.
\newblock In {\em Proceedings of the {Thirty}-{Second} {Conference} on
  {Uncertainty} in {Artificial} {Intelligence}}, {UAI}'16, pages 427--436,
  Arlington, Virginia, USA, June 2016. AUAI Press.

\bibitem[OR12]{orseau_space-time_2012}
Laurent Orseau and Mark Ring.
\newblock Space-{Time} {Embedded} {Intelligence}.
\newblock In Joscha Bach, Ben Goertzel, and Matthew Iklé, editors, {\em
  Artificial {General} {Intelligence}}, pages 209--218, Berlin, Heidelberg,
  2012. Springer.

\bibitem[WH25]{wyeth_hutter_semimeasures_2025}
Cole Wyeth and Marcus Hutter.
\newblock Value under ignorance in universal artificial intelligence.
\newblock (under review) 2025.

\bibitem[WHLT25]{wyeth_hutter_reflective_2025}
Cole Wyeth, Marcus Hutter, Jan Leike, and Jessica Taylor.
\newblock Limit-computable grains of truth for arbitrary computable
  extensive-form (un)known games.
\newblock (under review) 2025.
\end{thebibliography}

\appendix

\section{List of Notation}\label{app:Notation}

\par\vspace{0pt plus \textheight}
\begin{samepage}
\begin{tabbing}
  \hspace{0.13\textwidth} \= \hspace{0.73\textwidth} \= \kill
  {\bf Symbol }      \> {\bf Explanation}                                                    \\[0.5ex]
  $i,j,k\in\SetN$    \> index for natural numbers                                            \\[0.5ex]
  $[\![\text{bool}]\!]$ \> =1 if bool=True, =0 if bool=False                                 \\[0.5ex]
  $|\cX|\equiv\#\cX$      \> size of set $\cX$.                                                   \\[0.5ex]
  $\cA$               \> finite alphabet like $\{a,...,z\}$ or ASCII or $\{0,1\}$.           \\[0.5ex]
  $\cA^*,\cA^\infty$\> set of all (finite, infinite) strings  over alphabet $\cA$\\[0.5ex]
  $\SetB$         \> the binary set $\{0,1\}$                                                \\[0.5ex]
  $A \times B$         \> Cartesian product of sets $A$ and $B$ \\[0.5ex]
  $t,n\in\SetN$      \> time index, e.g.\ $x_{1:n}$ or $x_{<t}$                              \\[0.5ex]
  $x_{1:n}\in\cA^n$   \> string of length $n$                                                \\[0.5ex]
  $x_{<t}\in\cA^{t-1}$\> string of length $t-1$                                              \\[0.5ex]
  $\epstr$           \> empty string                                                         \\[0.5ex]
  \rule{1.4ex}{1.4ex}\> end of proof                                                         \\[0.5ex]
  $\SetR,\SetN,...$  \> set of real,natural numbers                                          \\[0.5ex]
  $\nu,\nu^\cdot$ \> semimeasure, chronological semimeasure                                          \\[0.5ex]
  $\mu,\mu^\cdot$ \> true distribution/environment \\[0.5ex]
  $\xi,\xi^\cdot$ \> mixture of distributions/environments \\[0.5ex]
  $\xi_U$ \> Solomonoff's universal mixture distribution \\[0.5ex]
  $\xi^\text{AI}$ \> Hutter's universal mixture environment \\[0.5ex]
\end{tabbing}
\end{samepage}

\end{document}